\documentclass[twocolumn]{svjour3}          % twocolumn
\smartqed  % flush right qed marks, e.g. at end of proof
\usepackage{graphicx}
\usepackage{hyperref}
\usepackage{enumitem} 
\usepackage{mathptmx}      % use Times fonts if available on your TeX system
\usepackage{graphicx}%
\usepackage{multirow}%
\usepackage{amsmath,amssymb,amsfonts}%
\usepackage{mathrsfs}%
\usepackage[title]{appendix}%
\usepackage{xcolor}%
\usepackage{textcomp}%
\usepackage{manyfoot}%
\usepackage{booktabs}%
\usepackage{algorithm}%
\usepackage{algorithmicx}%
\usepackage{algpseudocode}%
\usepackage{listings} 
\usepackage{array}
\usepackage{booktabs}
\usepackage{tabularray}
\usepackage{float}
\usepackage{caption}
\usepackage{subcaption}
\usepackage{pifont}
\usepackage{rotating}
\begin{document}

\title{LIB-KD: Teaching Inductive Bias for Efficient Vision Transformer Distillation and Compression. %\thanks{Grants or other notes
}
%\subtitle{Do you have a subtitle?\\ }


\author{Gousia Habib$^1$ \and Tausifa Jan Saleem$^1$ \and Ishfaq Ahmad Malik$^2$ \and Brejesh Lall$^1$ %etc.
}


\institute{$^1$ Department of Electrical Engineering \at
             Indian Institute of Technology Delhi \\
              \email{gousiya.cstaff@iitd.ac.in}           %  \\
           \and
      $^2$   Yogananda School of AI, Computers and Data Sciences  \at
          Shoolini Univesrity
}


\maketitle

\begin{abstract}
With the rapid development of computer vision, Vision Transformers (ViTs) offer the tantalising prospect of unified information processing across visual and textual domains due to the lack of inherent inductive biases in ViTs.  ViTs require enormous datasets for training. We introduce an innovative ensemble-based distillation approach that distils inductive bias from complementary lightweight teacher models to make their applications practical. Prior systems relied solely on convolution-based teaching. However, this method incorporates an ensemble of light teachers with different architectural tendencies, such as convolution and involution, to jointly instruct the student transformer. Because of these unique inductive biases, instructors can accumulate a wide range of knowledge, even from readily identifiable stored datasets, which leads to enhanced student performance. Our proposed framework \textbf{ LIB-KD} also involves precomputing and keeping logits in advance, essentially the unnormalized predictions of the model. This optimisation can accelerate the distillation process by eliminating the need for repeated forward passes during knowledge distillation, significantly reducing the computational burden and enhancing efficiency. Code for reproducibility will be available here \href{https://github.com/gousiya26-I/Distilling-Inductive-bias}{GitHub Repository}. \\
\textbf{Impact Statement:} Initially designed for natural language processing, transformers are a promising alternative to convolutional neural networks for visual learning. Nevertheless, their effectiveness falls when confronted with limited training data due to a lack of inherent inductive bias. This paper aims to bridge this gap and enhance their practical utility by developing an innovative ensemble-based distillation approach.

A single-channel distillation token facilitates a lightweight teacher ensemble with diverse inductive biases. In addition to imparting valuable inductive biases, this ensemble provides an efficient way of deploying these models on edge devices with limited computing power.
\keywords{Visual Transformers (VTs) \and Vision Transformers (ViTs) \and CNNs \and Involution \and INNs \and Knowledge Distillation \and KLD Loss}
%\subclass{MSC code1 \and MSC code2 \and more}
\end{abstract}
\section{Introduction}\label{sec1}
Visual Transformers (VTs) are becoming more popular in computer vision as an alternative to traditional convolutional neural networks (CNNs). A wide range of tasks can be performed using them, such as image classification \cite{1}, object detection \cite{3}, segmentation \cite{4}, image generation \cite{5}, and 3D data processing \cite{6}, among others. Having evolved from the Transformer model, the gold standard in Natural Language Processing (NLP), these architectures draw inspiration from the renowned model. ViTs offer the potential to create unified information-processing frameworks that span visual and textual fields. A groundbreaking contribution in this direction is the Vision Transformer (ViT). ViT divides an image into non-overlapping patches and then linearly transforms each patch into an input embedding, effectively creating a ``token'' of that image. Similarly to how tokens are processed in NLP transformers, all these tokens undergo a series of Multi-Head Self Attention (MHSA) given by equation \ref{eq1} and feed-forward layers. The Mathematical representation of MHSA is given as:
\begin{equation}
\text{Attention}(Q,K,V)=\text{Softmax}\left(QK^{T}/\sqrt{d}\right)V \label{eq1}
\end{equation}
Where $Q, K \text{ and } V$ represent Queries, Keys and Values, and $d$ represents the model depth of ViT. ViTs can leverage attention layers to model global relationships among tokens, differentiating them from CNNs. \\
Contrary to CNNs, where convolutional kernels' receptive fields limit how relationships can be learned, VTs provide a more expansive representation capability. Although VTs represent more information, they lack CNNs' inherent inductive biases, which decreases their representation power. These biases are derived from exploiting local information, translation invariance, and hierarchical data structures.\\
To achieve this trade-off, VTs typically require a substantial amount of training data, exceeding the data requirements of conventional CNNs. In contrast to ResNets, which possess similar model capacities, ViT's performance is noticeably inferior when trained on ImageNet-1K, a dataset that comprises approximately 1.3 million samples. ViTs rely on a larger dataset because they need to learn specific local characteristics of visual signals, something CNNs build into their architecture by design. The reason why ViTs require a large-scale dataset to understand inductive biases is illustrated in the CKA Similarity metric, given as:
\begin{equation}
  \text{CKA}(P,Q)=\frac{\text{HSIC}(P,Q)}{\sqrt{\text{HSIC}(P,P)\times \text{HSIC}(Q,Q)}}                           \label{eq2}
\end{equation}  
where, $P \in \mathbf{R}^{m \times p_1} \times\mathbf{R}^{p_1 \times m} $ and $Q \in \mathbf{R}^{m \times p_2} \times\mathbf{R}^{p_2 \times m}$ denote the Gram matrices for the two layers with $p_1$ and $p_2$ Neurons (which measure the similarity of a pair of data points according to layer representations).\\
The calculated representations using the CKA Similarity metric are depicted in Figure \ref{fig:1}. There is a marked difference between ViTs and CNNs in their representation structure, with ViTs having highly similar representations throughout the model. In contrast, ResNet models show much less similarity between lower and higher layers \cite{7}.
\begin{figure}[H]
\centering
         \includegraphics[width=\linewidth]{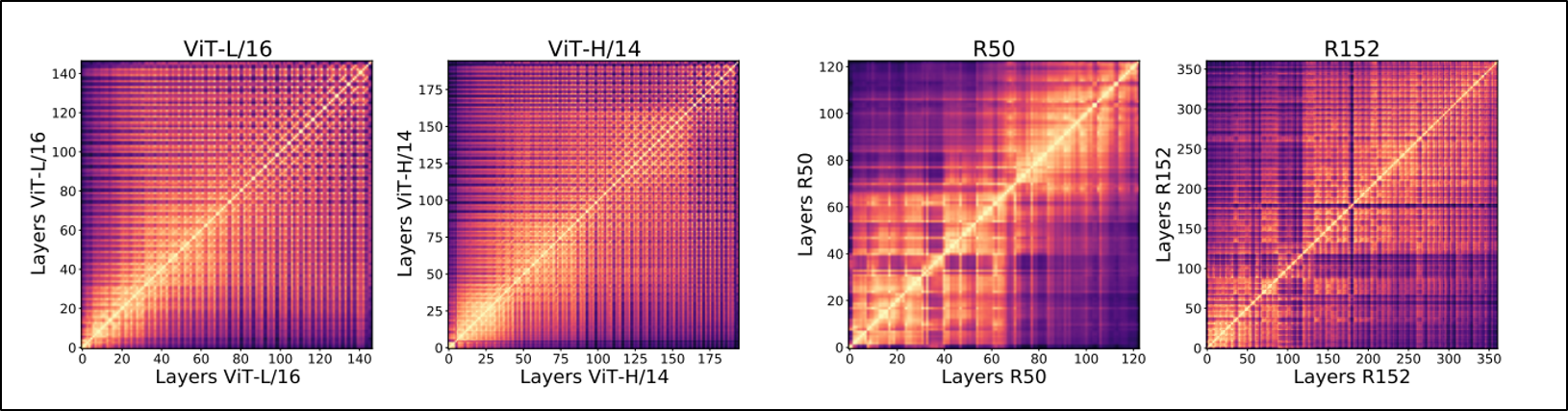}
        \caption{Representational Structure ViTs and CNNs using CKA Similarity \cite{7}.}
        \label{fig:1}
    % Second subfigure
 \end{figure}
From the above heatmaps depicted in figure \ref{fig:1}, it is clear that ViTs and CNNs, such as ResNet models, provide significant differences in representation. ViTs exhibit remarkable consistency in their representations across model layers. As we progress from lower to higher layers, the features extracted in ViTs remain similar. Heatmaps illustrate this uniformity when comparing the similarities between layers in different ViT models. Heatmaps show a grid-like pattern, with high similarity scores between adjacent and distant layers. Alternatively, ResNet models display their representation structure in another way. We observe distinct stages in the structure of ResNets when examining the similarity between layers. Compared to higher layers, lower layers are relatively less similar. The features extracted in the early layers of a ResNet model differ significantly from those removed in the later layers. The heatmap reflects this stage-wise dissimilarity, where we can see more miniature similarity scores between layers at different stages. Overall, ViT models consistently maintain their features from layer to layer, whereas ResNet models exhibit more pronounced variations as layers are raised. \\
It is essential to consider local receptive fields represented by Figure \ref{fig:3} when comparing CNNs and ViTs \cite{7}. In CNNs, the local receptive field defines how neurons or units in a particular layer are connected to a specific region or patch of the input image. Using regional connectivity, CNNs can capture spatial hierarchies and patterns. In CNNs, lower layers learn simple features like edges and textures, and higher layers learn progressively more complex patterns. Hierarchical approaches benefit from local receptive fields, which ensure locality, translation invariance, and small details in data. ViTs, on the other hand, take a different approach. Using non-overlapping patches, images are divided into small patches and transformed into tokens.\\
Despite having mechanisms like self-attention to capture global relationships between tokens, ViTs may need help capturing fine-grained local details as effectively as CNNs. This is because the local receptive field concept intrinsic to CNNs is not explicitly enforced in ViTs. ViTs learn local properties through their self-attention mechanisms, which require more data to achieve the same efficiency level as CNNs. When comparing the two approaches, it is essential to consider how CNNs and ViTs handle local information and the trade-offs between enforcing locality (CNNs) and relying on self-attention mechanisms (ViTs). When choosing the appropriate architecture for a specific application, these differences should be considered when assessing the performance of various computer vision tasks.
\begin{figure}[h]
    \centering
    \includegraphics[width=0.7\linewidth]{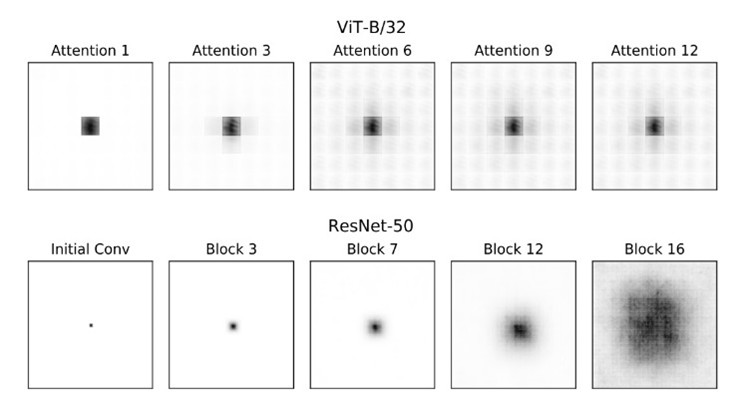}
    \caption{Effective receptive fields of ResNet are highly localised and grow gradually; ViT's are globalised \cite{7}.}
    \label{fig:3}
\end{figure}
Because local feature extraction is not explicitly enforced by an implicit bias and the prevalence of global receptive fields, ViTs need help being as efficient as CNNs. To unravel the full potential of Vision Transformers (ViTs) in computer vision \cite{8}, we must overcome the challenges of inductive bias and global receptive fields.\\
Innovative solutions must be explored and developed to increase the efficiency and adaptability of ViTs to eventually achieve performance levels comparable to or exceeding those achieved by CNNs across a broad range of visual recognition tasks. Dealing with inductive biases and global receptive fields in ViTs is a powerful motivation, propelling me on a scientific journey to devise and present groundbreaking solutions. This drive inspires me to take this challenge as an opportunity to create innovative and scientifically rigorous solutions that precisely address these intricacies.\\
We propose a solution based on co-advice, comprising homogeneous ensembles of either convolutional (Convo) or involutional (Invo) networks, as well as heterogeneous ensembles combining both Convo and Invo architectures \cite{9}. This approach not only addresses the challenge of inductive bias but also overcomes the computational and resource constraints associated with deploying Vision Transformers (ViTs) on edge devices. Our framework offers a comprehensive strategy that enhances model performance and adaptability while optimizing deployment efficiency on resource-limited edge computing platforms. The key contributions of our work are as follows.
The main contributions of our work are as follows.
\begin{enumerate}[label=\roman*.] % requires 
    \item We introduce a novel ensemble approach called \emph{ensemble-based cross-inductive bias distillation} \textbf{ LIB-KD} to transfer valuable knowledge into lightweight student models using complementary teacher models such as INNs and CNNs.
    \item This distillation technique enhances the performance and capabilities of vision transformer models by leveraging the unique strengths of each teacher model.
    \item Ensemble guidance is provided through a single distillation token, with the DeiT model serving as the foundational architecture.
    \item Instead of presenting INN and CNN models as separate tokens, we propose creating an ensemble that includes both, enabling the exploitation of their complementary inductive biases.
    \item To reduce computational complexity, we guide the ViT using a single distillation token, optimizing the knowledge transfer process while maintaining efficiency.
\end{enumerate}
To the best of our knowledge, this is the first work to apply transformers on downstream datasets by leveraging an ensemble of lightweight teacher models to effectively impart inductive bias. Beyond transferring inductive bias, our approach delivers a lightweight, resource-efficient solution that enables real-time deployment on small datasets while achieving superior generalization performance.\\
To substantiate these claims, we begin in Section 2 with a discussion of related work that situates our approach in the broader context. Section 3 introduces our proposed ensemble framework, explaining the integration of the Multi-Head Self-Attention Layer, Convolutional Filter, and Involution Filter. The underlying theoretical justification is provided in Section 4, followed by detailed experimental validation in Section 5. Section 6 benchmarks our approach against state-of-the-art (SOTA) methods, while Section 7 presents an ablation analysis to isolate each component’s impact. Finally, Section 8 concludes with key findings and potential research directions.
\section{Related Work} \label{Sec2}
\textbf{CNN:} The convolution operator was invented approximately three decades ago in \cite{11}. Since the advent of deep CNNs like AlexNet, VGGNet, ResNet, and EfficientNet \cite{13,15,14}, it has resurged and made a noticeable impact. As a result of these deep CNNs, we are witnessing a breakthrough in nearly any task imaginable. CNNs perform exceptionally well because of their inherent characteristics, called inductive biases, especially translation equivariance \cite{15} and spatial-agnostic properties \cite{16} associated with the convolution operator. It is only possible to capture spatially distant relationships in CNNs if deliberate efforts are made to increase the kernel size and model depth.\\
\noindent
\textbf{Transformers:} Recent attention has been paid to transformers in computer vision, which originated in natural language processing \cite{17}. As reported in \cite{2}, the ViT feeds $16\times16$ image patches into a standard transformer, achieving comparable results as CNNs on JFT-300M \cite{2}. However, its superiority comes at the expense of an enormous amount of labelled data and a lengthy training period. Moreover, ViTs do not achieve significant accuracy improvements when insufficient data is provided. Furthermore, DETR and ViT were proposed in \cite{18}. When ViT \cite{19} represents images as semantic tokens and exploits transformers in image classification and semantic segmentation, DETR \cite{18} uses bipartite matching loss and a transformer-based encoder-decoder structure. Besides the application, as mentioned earlier, it has been theoretically demonstrated that transformers use self-attention mechanisms as expressive as convolution layers \cite{20}.\\
\noindent
\textbf{INNs:} Unlike the convolution operator, the Involution operator was introduced relatively recently in \cite{23}. Contrary to a convolution operator, an involution kernel shares its spatial extent across channels but is spatially agnostic. When compared to convolution, involution exhibits precisely the opposite inherent characteristics. Consequently, involution can capture spatial relationships over a long distance. RedNet architectures, which use involution to achieve enhanced performance, are consistently superior to CNNs and transformers, as shown in \cite{24}.\\ 
\noindent
\textbf{Knowledge Distillation (KD):} KD is a model compression technique that uses a high-capacity teacher model to train lightweight student models \cite{27}. According to the original formulation by \cite{2}, this objective is achieved by minimising the Kullback-Leibler (KL) divergence between student and teacher probabilistic predictions. Since then, KD has been applied to many learning tasks, such as privileged learning \cite{29,27}, cross-modal learning \cite{26}, adversarial learning \cite{31}, contrastive learning \cite{24}, and incremental learning \cite{32}. The token-based KD strategy proposed by \cite{2} fits with the context of our research. As a result of distilling knowledge from a powerful ensemble of CNN and INN-based teachers, DeiT \cite{2} performed equivalently to CNNs, whereas the earlier ViT \cite{33} did not consider tiny datasets \cite{35,36,34}.\\
\noindent
\textbf{ViTs for small datasets:} Researchers presented an effective strategy for training Vision Transformers (ViTs) without the need for large-scale pretraining datasets in this study \cite{37}. The authors in \cite{37} employed a self-supervised inductive bias learning approach directly from these modest datasets. Self-supervised learning initialises the network, followed by supervised training on the same dataset to fine-tune it \cite{40,38}.\\
Transformers are becoming increasingly valuable in various fields because of the success of ViTs. Due to their inability to capture local information, ViTs are limited when trained directly on small datasets. A hybrid model combining ViTs and CNNs is proposed in this \cite{41} work to address this issue. As part of the transformer architecture, this model incorporates convolutional operations that enhance classification performance on small datasets, specifically a novel convolutional parameter sharing multi-head attention (CPSA) block and a local feed-forward network (LFFN) block. The authors in \cite{41} showed state-of-the-art results on small datasets, demonstrating a promising avenue for leveraging transformers.\\
\noindent
\textbf{Lightweight ViTs for small datasets:} Lightweight CNNs have proved invaluable in various mobile vision tasks. A recent effort has been made to create lightweight, efficient ViTs. MobileViT \cite{42} outperformed MobileNets \cite{43} and ShuffleNet \cite{44} by combining standard convolutions and transformers. Based on neural architecture search (NAS) \cite{45}, the researchers identified a range of efficient ViTs with varying computational requirements, outperforming existing benchmarks. The model throughput efficiency of ViTs was enhanced by \cite{50,46} by optimising the speed of inference for small to medium-sized ViTs.\\
In contrast, our methodology \textbf{LIB-KD} emphasises imbuing inherent inductive biases from a diverse ensemble of lightweight teachers into ViTs. The primary objective is to enhance the efficiency of ViTs, making them competitive with CNNs while utilising fewer parameters and mitigating computational complexities and resource requirements. Simultaneously, we aim to optimise these ViTs for deployment in resource-constrained edge computing environments \cite{51}.
\section{Proposed Ensemble  Approach for Imparting Cross-Inductive Bias to ViTs via KD} \label{sec3}
According to our hypothesis, our teachers acquire distinct knowledge despite being trained on the same dataset due to inherent inductive biases, spatial-agnostic and channel-specific in convolution and spatial-specific and channel-agnostic in involution. Therefore, teachers with different inductive biases offer different perspectives and make different assumptions about data. However, ResNet-26 and ResNet-38 \cite{13}, which have similar inductive biases but various performances, describe data relatively similarly. Based on the complementary inductive biases of these different types of teachers, our method only requires two highly efficient teachers (a CNN and an INN), both of which can be easily trained. During Distillation, these teachers' knowledge complements one another, resulting in increased accuracy in the student transformer.\\
When pretraining small models directly on extensive data, they produce few benefits, especially when transferring them to downstream tasks. To solve this problem, we implement knowledge distillation to maximise the benefits of pretraining for small models. We emphasise Distillation before training instead of prior approaches that emphasise Distillation during the fine-tuning stage. In addition to allowing small models to learn from larger-scale models, this technique also improves downstream performance.
\begin{figure*}
    \centering
    \includegraphics[width=0.9\linewidth]{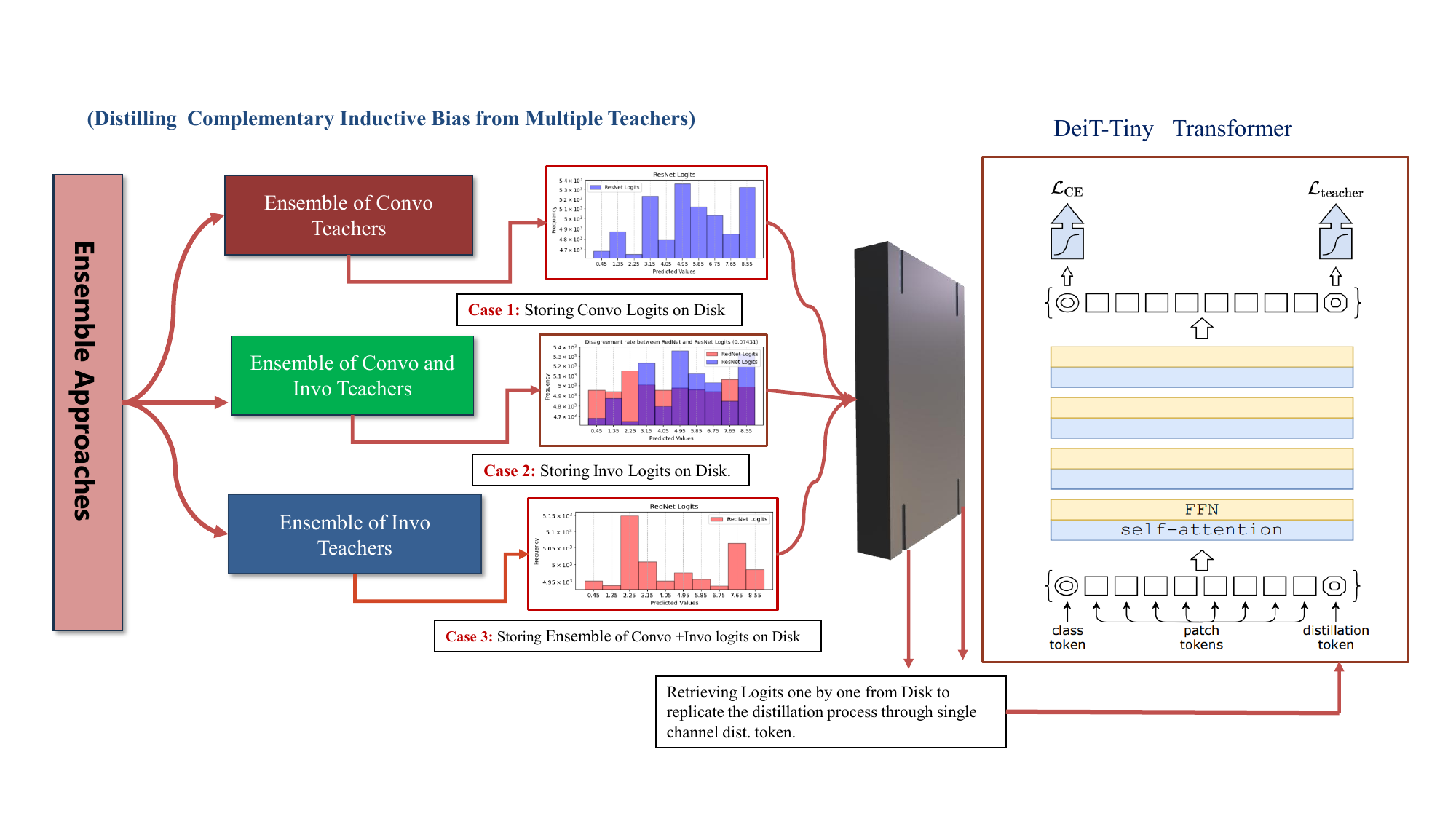}
    \caption{Novel pretraining distillation framework utilising a single channel distillation token. Storing and retrieving logits from an ensemble of Invo, Convo and a mixture of both Convo and Invo teachers for distillation.}
    \label{fig:4}
\end{figure*}
In contrast, conventional pretraining using Distillation is wasteful and resource-intensive. In every iteration, most computing resources are spent on passing training data through the large teacher model rather than training the minor target student. Additionally, a prominent teacher model consumes substantial GPU memory, slowing down the training of the target student due to batch sizes. To address these challenges, we propose a unique and fast distillation framework \textbf{ LIB-KD} (Fig. \ref{fig:4}). By storing teacher predictions in advance, we can replicate the distillation process during training without performing extensive forward computations or allocating memory to the large teacher model. The various crucial components of our proposed methodology are discussed as follows:
\subsection{Multi-Head Self-Attention Layer} Considering $M$ input feature vectors $\{ \textbf{x}_m\in \mathbf{R}^{d_i}\mid m=1,2, \cdots, M  \}$, an array of rows is stacked in matrix form $Y\in \mathbf{R}^{M \times d_i} $. As part of a single-head self-attention layer, a query, key, and value matrix is calculated according to the following points:
\begin{align*}
 Q&=YM^{q} \in \mathbf{R}^{M \times dep_q}, 
 K=YM^k\in \mathbf{R}^{M \times dep_k} \\
 V&=YM^{v}\in \mathbf{R}^{M \times dep_v}   \label{eq3}
\end{align*}
where $M^{q} \in   \mathbf{R}^{d_i \times d_k}, M^k \in \mathbf{R}^{d_i \times d_k}$ and $M^v\in \mathbf{R}^{d_i\times d_v}$ represent different learnable parameters of the model. The outcome of the attention model is given as: 
\begin{equation}
\text{Self-Attention}(Q,K,V)=\text{Softmax}\left( \frac{QK^T}{\sqrt{dep_k}} \right)V    
\end{equation}
Each row of this matrix is fitted with the Softmax function. An ensemble of independent self-attention layers is the foundation for a multi-head self-attention layer.
\subsection{Convolutional Filter} A group of fixed-sized convolution filters, each with a size of $K \times K$ represented by $F_k \in \mathbf{R}_{c \times \text{x} \times K}$, $k=1,2, \cdots, C_o$ containing $C_i$ convolution kernels $F_{kc}=\mathbf{R}_{k \times k}$, $c=1,2, \cdots, C_i$. Their kernels perform the scalar addition-multiply operation on the incoming feature map in a sliding window fashion to produce the output feature vector $Y\in\mathbf{R}_{H \times W\times c} $ depicted as 
%\begin{tiny}
    \begin{equation}
    Y_{i,j,k}=\sum_{c=1}^{C_i}\sum_{(u,v)\in \Delta_k}\mathcal{F}_{k,c,u+\lfloor K/2 \rfloor, v+\lfloor K/2 \rfloor}X_{i+u, j+v,c} \label{eq5}
\end{equation}
%\end{tiny}
Where $u$ and $v$ represent the spatial offsets in the $K\times K$ neighbourhood, and $\Delta_K \in\mathbf{Z}^2 $ describes the balances around the centre pixel, considering the convolution held on it.
%\begin{tiny}
    \begin{equation}
 \Delta_K = \begin{bmatrix} -\lfloor K/2\rfloor, \cdots , \lfloor K/2\rfloor \end{bmatrix} \times \begin{bmatrix} -\lfloor K/2\rfloor, \cdots , \lfloor K/2\rfloor \end{bmatrix} \label{eq6}
\end{equation}
%\end{tiny}
Furthermore, depthwise convolution \cite{42} propels the formula to the group convolution \cite{43} to the end, where each convolution filter is applied on a single feature channel. So, equation \ref{eq5}, $\mathcal{F}_k$ is replaced by $\mathcal{G}_k$ and the formula is rewritten as 
%\begin{tiny}
    \begin{equation}
    Y_{i,j,k}= \sum_{(u,v)\in \Delta_k}\mathcal{G}_{k,c,u+\lfloor K/2 \rfloor, v+\lfloor K/2 \rfloor}X_{i+u, j+v,c} \label{eq7}
\end{equation}
%\end{tiny}
Where $\mathcal{G}_k$ represents channel-wise $k$th feature slice about $x$th feature input.
\subsection{Involution Filter}
In comparison to the above-discussed standard or group convolution. Involution kernels \cite{20} $H\in \mathbf{R}^{H\times W\times K\times K\times G}$. They are devised to invert the inherent characteristics of the standard convolution (spatial agnostic and channel-specific) into (spatial typical and channel agnostic) behaviour. The output feature vector produced by applying such involution kernels on the input feature map yields output as:
%\begin{tiny}
    \begin{equation}
    Y_{i,j,k}= \sum_{(u,v)\in \Delta_k}H_{i,j,u+\lfloor K/2 \rfloor, v+\lfloor K/2 \rfloor, \lceil kG/C \rceil}X_{i+u, j+v,k} \label{eq8}
\end{equation}
%\end{tiny}
Besides convolution kernels, which use a constant-size seed. Involution kernel H utilises a variable-size kernel based on the input feature map. $(i,j)$. Involution kernels could be generated based on (part of) the original input tensor so that the output kernels align comfortably with the input. Kernel generation is symbolised, and functional mapping is abstracted as: 
\begin{equation}
    \mathcal{H}_{i,j}=\phi\left(H_{\Psi_{i,j}} \right) \label{eq9}
\end{equation}
Where $\Psi_{i,j}$ represents a group of pixels. The overall learning objective of our proposed framework is a weighted sum of two losses: base loss and cross-entropy (CE). Base loss is minimised between ground truth labels with the student, and CE between dist token and hard label teacher predictions (Ensemble of CNN and INN). 
%\begin{tiny}
    \begin{align*}
   & Z_t = \text{argmax}[\mathcal{L}_{CE}(\sigma(z_{t_1}),Y)+\mathcal{L}_{CE}(\sigma(z_{t_1}),Y/2)] \\
    \text{Overall Loss}&=\mathcal{L}_{CE}(\sigma(z_{s}),Y)+CE\left[ \left(\frac{z_s}{\tau_1} \right), \left(\frac{z_t}{\tau_1} \right) \right]
\end{align*}
%\end{tiny}
where $z_t$ represents the ensemble of hard-predicted labels of two complementary teachers, CNN and INN.
\section{Theoretical Proof of Proposed Approach}
\begin{proposition}
Let \(T_{\text{CNN}}\) and \(T_{\text{INN}}\) be complementary teachers and \(S_{\text{ViT}}\) a student ViT. Training \(S_{\text{ViT}}\) via KD from their ensemble (weighted KL on logits) yields an inductive bias combining local features from \(T_{\text{CNN}}\) and global transformations from \(T_{\text{INN}}\), surpassing either teacher alone.
\end{proposition}
\begin{proof}
Define the output logits of the teacher and student models for an input \( x \) as:
\[
F_{\text{CNN}}(x), \quad F_{\text{INN}}(x), \quad F_{\text{ViT}}(x) \in \mathbf{R}^C,
\]
where \( C \) is the number of classes. The corresponding softmax probability vectors are:
%\begin{tiny}
\begin{align*}
    p_{\text{CNN}}(x) &= \text{softmax}(F_{\text{CNN}}(x)),
     p_{\text{INN}}(x) = \text{softmax}(F_{\text{INN}}(x)),\\ 
     p_{\text{ViT}}(x) &= \text{softmax}(F_{\text{ViT}}(x))
\end{align*}
%\end{tiny}
\textbf{Step 1: Formulation of Ensemble Knowledge Distillation Loss:} The ensemble KD loss minimized during training is:
%\begin{tiny}
    \[
L_{\text{KD}} = \lambda_1 \, \text{KL}\big(p_{\text{CNN}}(x) \| p_{\text{ViT}}(x)\big) + \lambda_2 \, \text{KL}\big(p_{\text{INN}}(x) \| p_{\text{ViT}}(x)\big)
\]
%\end{tiny}
where \(\lambda_1, \lambda_2 > 0\) control the contribution of each teacher.\\
\textbf{Step 2: Interpretation of Individual Loss Terms}
\begin{itemize}
    \item \textit{CNN teacher loss term}:
    \[
    L_{\text{CNN}} = \text{KL}\big(p_{\text{CNN}}(x) \| p_{\text{ViT}}(x)\big)
    \]
    enforces the student to mimic the CNN’s local spatial feature extraction, encoding translation invariance and locality.
    \item \textit{INN teacher loss term}:
    \[
    L_{\text{INN}} = \text{KL}\big(p_{\text{INN}}(x) \| p_{\text{ViT}}(x)\big)
    \]
    enforces the student to preserve the global bijective transformations and long-range dependencies learned by the INN.
\end{itemize}
\textbf{Step 3: Why Ensemble Distillation Leads to Stronger Inductive Bias}\\
\textbf{Claim:} The inductive bias \( \mathcal{B}_{\text{ensemble}} \) learned by distilling from the ensemble is \textit{strictly stronger} than biases from distilling from either teacher alone:
\[
\mathcal{B}_{\text{ensemble}} \supset \mathcal{B}_{\text{CNN}}, \quad \mathcal{B}_{\text{ensemble}} \supset \mathcal{B}_{\text{INN}}.
\]
\textit{3.1. Ensemble KD Loss as Multi-Objective Optimization}\\
Minimizing
\[
L_{\text{KD}} = \lambda_1 L_{\text{CNN}} + \lambda_2 L_{\text{INN}}
\]
means finding \( p_{\text{ViT}}(x) \) that approximates both \( p_{\text{CNN}}(x) \) and \( p_{\text{INN}}(x) \), i.e., solving
\[
p_{\text{ViT}}^* = \arg\min_p \left[\lambda_1 \text{KL}(p_{\text{CNN}} \| p) + \lambda_2 \text{KL}(p_{\text{INN}} \| p) \right].
\]
\textit{3.2. Geometric View: Intersection of Hypothesis Spaces}\\
The hypothesis sets corresponding to each teacher’s bias satisfy:
\[
\mathcal{B}_{\text{ensemble}} = \mathcal{B}_{\text{CNN}} \cap \mathcal{B}_{\text{INN}},
\]
meaning the student must lie in the intersection, encoding both local and global inductive biases simultaneously.\\
\textit{3.3. Convexity and Weighted Geometric Mean}\\
The solution to the above optimization is the weighted geometric mean:
\[
p_{\text{ensemble}}(x) \propto p_{\text{CNN}}(x)^{\lambda_1} \cdot p_{\text{INN}}(x)^{\lambda_2},
\]
which differs from each individual teacher distribution, representing a hybrid, more expressive distribution.\\
\textbf{Step 4: Consequences on Inductive Bias and Generalization}
\begin{itemize}
    \item The student inherits \textit{local pattern recognition} bias from the CNN.
    \item It simultaneously inherits \textit{global transformation} and \textit{information preservation} bias from the INN.
    \item This hybrid inductive bias leads to improved generalization and richer representations.
\end{itemize}
\end{proof}
%\begin{tiny}
    \section*{Summary}
\fbox{%
  \parbox{\dimexpr\linewidth-2\fboxsep-2\fboxrule\relax}{%
    \[
    \begin{aligned}
     \text{En}&\text{semble KD Loss} \\
    &= \underset{p}{\min} \ \lambda_1 \, \text{KL}(p_{\text{CNN}} \| p) + \lambda_2 \, \text{KL}(p_{\text{INN}} \| p) \\
       p_{\text{ViT}} &\approx p_{\text{ensemble}} \\
    & := \arg\min_p \left[ \lambda_1 \, \text{KL}(p_{\text{CNN}} \| p) + \lambda_2 \, \text{KL}(p_{\text{INN}} \| p) \right], \\
    & \text{Student inherits a strictly richer inductive bias } \\
    &\mathcal{B}_{\text{ensemble}} = \mathcal{B}_{\text{CNN}} \cap \mathcal{B}_{\text{INN}}, \\
    & \text{Improved generalization and stronger representational} \\ & \text{power compared to individual KD}.
    \end{aligned}
    \]
  }%
}
%\end{tiny}
By training the ViT with this integrated distillation loss, the student learns both local features from CNNs (spatial locality, translation invariance) and global transformations from INNs (information preservation), enhancing the ViT's inductive bias and improving generalization across vision tasks. This combined inductive bias enhances generalization on vision tasks.
\section{Experimental Results}
This section aims to understand our approach through a series of analytical experiments comprehensively. We begin by explaining the complexities of our distillation strategy, outlining the key steps and methodologies—our distillation strategy functions as a bridge between complex neural networks and simplified representations in our research. An extensive process of transferring knowledge (learned inductive bias) from a larger, more complicated model to a smaller, more efficient one is involved. Afterwards, we will examine three fundamental architectural paradigms in computer vision: CNNs, INNs, and ViTs, in a comparative analysis. During our exploration, we seek to gain a deeper understanding of the strengths and weaknesses of these approaches. 

The next crucial step is to discuss the configuration of the proposed model used during the training regime. As a result of this step, we can provide a clear understanding of the experimental setup and ensure the reproducibility of the results. Here, we will provide an overview of the meaning of hyperparameters in the context of deep learning experiments. As they govern various aspects of the training process, hyperparameters are crucial in shaping neural networks' behaviour and performance. These parameters include learning rates, batch sizes, weight initialisations, regularisation techniques, and optimisation algorithms. 

Besides hyperparameters, providing insight into the dataset used for conducting experiments is essential. This study used the CIFAR-100 dataset \cite{53}, a well-known computer vision and deep learning benchmark. CIFAR-100 is an excellent testbed for assessing the performance of various machine learning and deep learning models. There are 60,000 images in the CIFAR-100 dataset, divided into ten distinct classes, each with 6,000 images of size $32\times 32$. Furthermore, the dataset is divided into two subsets: the training set has 50,000 images, while the test set has 10,000 images. The hyper-parameters for training two complementary teacher models, CNN and INN, are given in Table \ref{tab1}.
\begin{tiny}
\begin{table*}
\caption{Hyperparameter details of the Proposed Method}\label{tab1}
\centering 
\begin{tabular}{|m{0.050\linewidth}|m{0.070\linewidth}|m{0.024\linewidth}|m{0.024\linewidth}|m{0.024\linewidth}|m{0.024\linewidth}|m{0.024\linewidth}|m{0.024\linewidth}|m{0.024\linewidth}|m{0.028\linewidth}|m{0.024\linewidth}|m{0.068\linewidth}|m{0.068\linewidth}|m{0.25\linewidth}|} 
\hline 
Model & \multicolumn{12}{m{0.382\linewidth}|}{Parameters of the Models} & Other Hyper
  parameters \\ 
\hline 
\multirow{2}{*}{INN} & \begin{sideways}\# stages\end{sideways} & \begin{sideways} Reduction Ratio\end{sideways}  & \begin{sideways}In channels \end{sideways} &\begin{sideways} Base Channels \end{sideways} &\begin{sideways} Stem channels \end{sideways} & \begin{sideways} Group Channels \end{sideways} & \begin{sideways} depth \end{sideways} &\begin{sideways} expansion \end{sideways}& \begin{sideways}Frozen Stages\end{sideways} &\begin{sideways} Out indices \end{sideways} & \begin{sideways} Stride \end{sideways} & \begin{sideways} Dilations \end{sideways} &  kernel
  size =1, label smooth= 0.1, optimizer = SGD, Initial learning rate =0.1,
  momentum =0.9, weight decay =5e-4, batch\_size =128, loss = CE  \\ 
\cmidrule{2-13}
 & 4 [1,2,4,1] & 4 & 3 & 64 & 64 & 16 & 26 & 4 & -1 & 3 & (1,2,2,2) & (1,1,1,1) &  \\ 
\hline
\multirow{2}{*}{CNN} &\begin{sideways} \# stages \end{sideways} & \begin{sideways} In channels \end{sideways}& \begin{sideways} Base Channels  \end{sideways}&\begin{sideways} depth \end{sideways}& \begin{sideways} Expansion \end{sideways}& \begin{sideways} Loss \end{sideways} & \begin{sideways} Kernel size \end{sideways} & \begin{sideways} Momentum  \end{sideways}& \begin{sideways} optimizer \end{sideways} & \begin{sideways} Learning rate \end{sideways} & \begin{sideways} Stride \end{sideways}& \begin{sideways} Weight decay  \end{sideways} & label smooth=0.1,
  batch\_size=128  \\ 
\cmidrule{2-13}
 & 4 [1,2,4,1] & 3 & 64 & 26 & 4 & CE & 1 & 0.9 & SGD & 0.1 & (1,2,2,2) & 5E-04 &  \\ 
\hline
\begin{sideways} Student Baseline \end{sideways} & \begin{sideways} optimizer \end{sideways} &\begin{sideways} Patch size \end{sideways} & \begin{sideways} Depth \end{sideways} & \begin{sideways} Mlp\_ratio \end{sideways} & \begin{sideways} Loss function \end{sideways} & \begin{sideways} Emb-dim \end{sideways}& \begin{sideways} Drop path rate \end{sideways} & \begin{sideways} \# heads  \end{sideways} & \begin{sideways} Learning rate \end{sideways} &\begin{sideways} Batch\_size \end{sideways} & \begin{sideways} Eps \end{sideways} & \begin{sideways}  Weight decay  \end{sideways} & Distillation
  type= hard, distillation-alpha=0.5, mixup=0, cut\_mix=0, mixup-prob=0, repeated\_aug=False, color-jitter=0, random
  erase =0 \\ 
\cmidrule{2-13}
\begin{sideways}
    KD
\end{sideways} & AdamW & 4 & 12 & 4 & CE & 192 & 0.1 & 3 & 128 & 0.05 & 1E-06 & 5E-04 &  \\ 
\hline
\begin{sideways} KD Student Superior \end{sideways}  & AdamW & 4 & 12 & 4 & CE & 192 & 0.1 & 3 & 128 & 0.05 & 1E-06 & 5E-04 & Distillation
  type= hard, distillation-alpha=0.5, mixup=0.8, cut\_mix=1.0, mixup-prob=1.0,
  repeated\_aug=True, colour-jitter=0.3, and random erase =0.25 \\
\hline 
\end{tabular}
\end{table*}
\end{tiny} 
It is worth noting that when it comes to model distillation, ``hard distillation'' and ``soft distillation'' refer to different techniques used to transfer knowledge to a smaller or student model from a larger model. Depending on your application, you may choose between hard and soft Distillation. Complex Distillation is well-suited for DeiT models \cite{8}. So, our distillation procedure is based on discrete hard labels of teacher models. The algorithm for the proposed framework \textbf{LIB-KD }is given in Algorithm \ref{algo1}.
\begin{algorithm}
\caption{LIB-KD Algorithm} \label{algo1}
\textbf{Input:} Ensemble of complementary teacher models, student model \\
\textbf{Output:} Optimized Lightweight Student Model \\
\textbf{Objective:} Distilling inductive bias via Knowledge Distillation (KD) \\ 
\vspace{1em}
\textbf{Step 1:} \\
\hspace*{1.5em} \textbf{Configure and Train Teacher Model 1} \\
\hspace*{3em} CNN Teacher $= \text{ResNet26}()$ \\
\hspace*{3em} $\text{CNN\_Loss} = \mathcal{L}_{\text{CE}}(\sigma(z_{t1}), Y)$
\vspace{0.5em}
\textbf{Step 2:} \\
\hspace*{1.5em} \textbf{Configure and Train Teacher Model 2} \\
\hspace*{3em} INN Teacher $= \text{RedNet26}()$ \\
\hspace*{3em} $\text{INN\_Loss} = \mathcal{L}_{\text{CE}}(\sigma(z_{t2}), Y)$
\vspace{0.5em}
\textbf{Step 3:} \\
\hspace*{1.5em} \textbf{Create an Ensemble of Teacher Models (Soft Averaging)} \\
\hspace*{3em} Ensemble Output $= \frac{1}{2}[\sigma(z_{t1}) + \sigma(z_{t2})]$ \\
\hspace*{3em} Prediction $= \arg\max\left(\frac{1}{2}[\sigma(z_{t1}) + \sigma(z_{t2})]\right)$
\vspace{0.5em}
\textbf{Step 4:} \\
\hspace*{1.5em} \textbf{Train Student Model from Scratch (Baseline)} \\
\hspace*{3em} Student Model $= \text{DeiT-Tiny}()$ \\
\hspace*{3em} $\text{Student\_Loss} = \mathcal{L}_{\text{CE}}(\sigma(z_s), Y)$
\vspace{0.5em}
\textbf{Step 5:} \\
\hspace*{1.5em} \textbf{Distill Inductive Bias using Ensemble Hard Labels} \\
\hspace*{3em} Total Loss Objective: \\
\hspace*{4em} $\min \left\{ \mathcal{L}_{\text{CE}}(\sigma(z_s), Y) + \mathcal{L}_{\text{CE}}\left( \frac{z_s}{\tau_1}, \frac{z_t}{\tau_1} \right) \right\}$ \\
\hspace*{3em} Base Student Loss: $\mathcal{L}_{\text{CE}}(\sigma(z_s), Y)$ \\
\hspace*{3em} Distillation Loss: $\mathcal{L}_{\text{CE}}\left( \frac{z_s}{\tau_1}, \frac{z_t}{\tau_1} \right)$
\vspace{0.3em}
\hspace*{1.5em} Optionally, apply weighted loss: \\
\hspace*{3em} $\min \left\{\alpha \cdot \mathcal{L}_{\text{CE}}(\sigma(z_s), Y) + (1 - \alpha) \cdot \mathcal{L}_{\text{CE}}\left( \frac{z_s}{\tau_1}, \frac{z_t}{\tau_1} \right) \right\}$
%\begin{algorithmic}[1] 
%\State \textbf{Step 1: Configure and Train Teacher Model 1}
%\State \hspace{1em} CNN Teacher $=$ ResNet26()
%\State \hspace{1em} CNN\_Loss$ = \mathcal{L}_{\text{CE}}(\sigma(z_{t1}), Y)$
%\Require $n \geq 0 \vee x \neq 0$
%\Ensure $y = x^n$ 
%\State $y \Leftarrow 1$
%\If{$n < 0$}\label{algln2}
%        \State $X \Leftarrow 1 / x$
%        \State $N \Leftarrow -n$
%\Else
%        \State $X \Leftarrow x$
%        \State $N \Leftarrow n$
%\EndIf
%\While{$N \neq 0$}
%        \If{$N$ is even}
%            \State $X \Leftarrow X \times X$
%            \State $N \Leftarrow N / 2$
%        \Else[$N$ is odd]
%            \State $y \Leftarrow y \times X$
%            \State $N \Leftarrow N - 1$
%        \EndIf
%\EndWhile
%\end{algorithmic}
\end{algorithm}
In the preceding section, we present the findings of a study that looks at how an ensemble of complementary teacher models can be used to distil inductive bias guided through Single-Channel Distillation Tokens (dist.). Table 4 provides a detailed overview of the performance metrics we obtained during the experimentation. Our study reports top-1 and top-5 test accuracies, both for the baseline hyperparameters and for a superior distillation technique. 

Our distillation process differs from conventional approaches, using a single-channel distillation token rather than separate convolutional (convo) and invertible (invo) tokens. Also, streamlining the model architecture using single-channel distillation tokens and related techniques significantly reduces parameter count. As a result, we can achieve computational efficiency and demonstrate the efficacy of our distillation strategy under resource constraints. Furthermore, we compare the results obtained with and without data augmentation techniques to assess their impact on Distillation. As a result of this strategic choice, the distillation procedure is less likely to experience computational overload. Tables \ref{tab:2} and \ref{tab:3} report the results obtained by training three different initialisations of ResNet and RedNet models used for ensemble distillation of complementary inductive to lightweight student models. The loss and accuracy curves of three different initialisations of RedNet 26 models are given in Figures \ref{fig:5}a, \ref{fig:5}b, \ref{fig:5}c. 
\begin{figure}     \centering   \includegraphics[width=\linewidth]{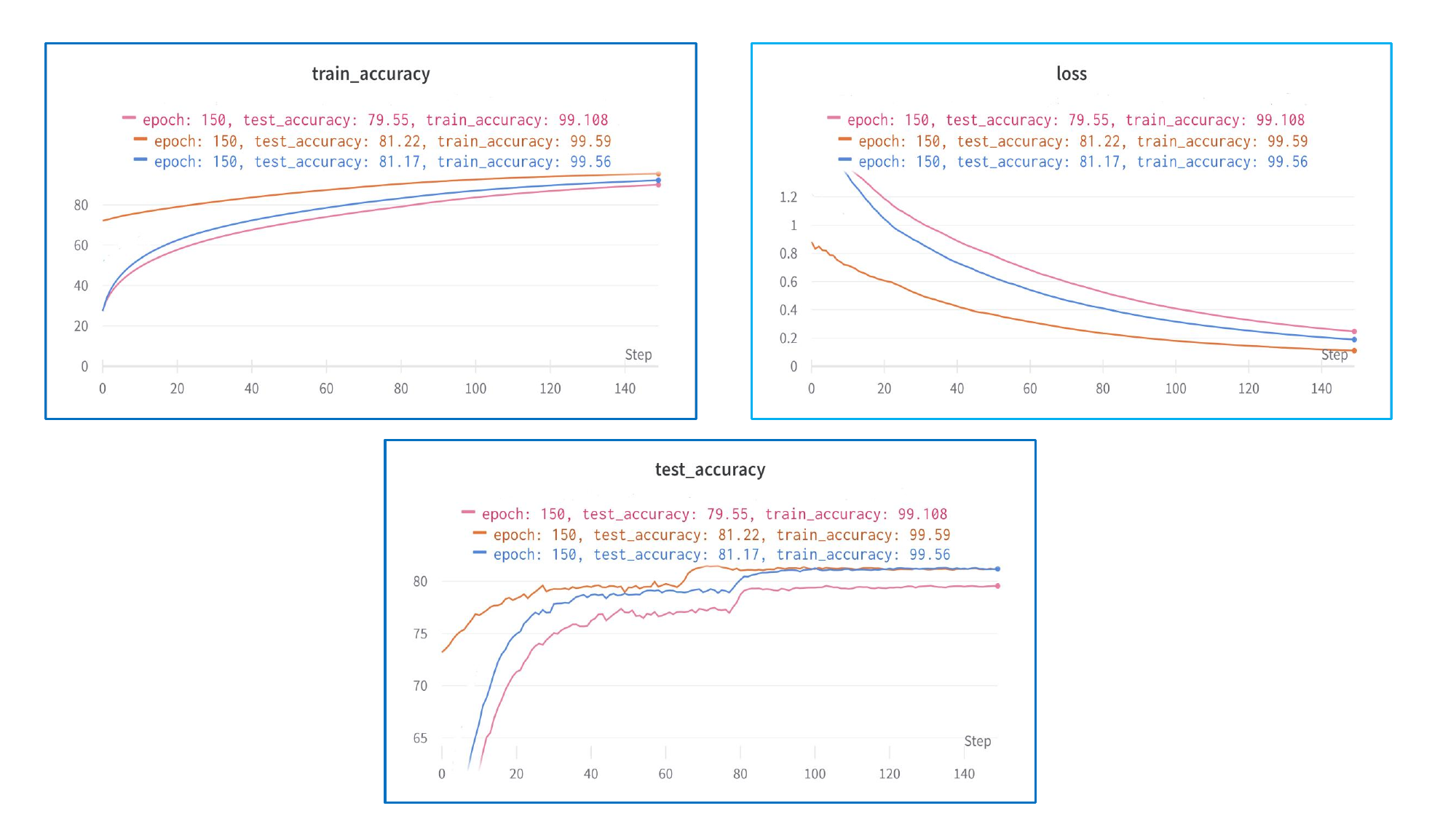} \caption{a) Training accuracy  Curves. b) Training  Loss curves. c)Test accuracy Curves.} \label{fig:5} \end{figure}
\noindent 
A comparative analysis of the initialisations of RedNet is presented in the following Table \ref{tab:2}. Several parameters are analysed within each variant, and the Top-1 and Top-5 test accuracy for each variant.
\begin{tiny}
\begin{table*}
    \centering
        \caption{RedNet initialisations} \label{tab:2}
    \begin{tabular}{|m{1.1cm}|m{0.7cm}|m{1.5cm}|m{2.5cm}|m{1.6cm}|m{1.9cm}|m{2.5cm}|m{1.9cm}|}
    \hline 
    Model&	Depth&	Initializations&	Ensemble Technique	& \# Parameters &	Parameter Size	& Top 1 Test Accuracy &	Train Accuracy   \\
    \hline 
   RedNet	& 26 & 	1	   & - &	 7M  & 	27.41MB&	\textbf{81.17\%} &	99.59\%  \\ 
   \hline 
   RedNet	&26 &2	     & - &	 	7M&	27.41MB & \textbf{79.55\%} & 	99.56\% \\
   \hline 
   RedNet &	26&	3&	- &	 7M& 27.41MB&	\textbf{81.22\%} &	99.11\% \\
   \hline 
   Ensemble RedNet&	3(26)  &	- &Majority Voting&	 	21M&	82.23MB&	\textbf{84.14\%}	&99.89\% \\
   \hline 
   Ensemble RedNet	& 3(26) &	- &Soft averaging&	 	21M &	82.23MB &	\textbf{84.14\%}	& 99.77\%\\ 
   \hline 
    \end{tabular}
\end{table*}
\end{tiny}
\noindent 
Also, we provide a comprehensive comparison of various ResNet variants in the following Table \ref{tab:3} 
\begin{table*}
    \centering
        \caption{ResNet initialisations}     \label{tab:3}
    \begin{tabular}{|m{1.1cm}|m{0.7cm}|m{1.5cm}|m{2.5cm}|m{1.6cm}|m{1.9cm}|m{2.5cm}|m{1.9cm}|}
 \hline 
 Model&	Depth&	Initializations&	Ensemble Technique	& \# Parameters &	Parameter Size	& Top 1 Test Accuracy &	Train Accuracy         \\
 \hline 
    ResNet &	26 &	1&	   - &	 9M& 33.34MB &	\textbf{87.01\%}	&99.14\% \\ 
    \hline 
    ResNet	&26&	 2&	    - &	 9M&	33.34MB&	\textbf{87.89\%}	&99.19\% \\ \hline 
ResNet	& 26 &	 3&		- & 9M &33.34MB&	\textbf{89.66\%}	&99.15\% \\ \hline 
Ensemble ResNet& 3(26)&	- &	Majority Voting&	27M&	100.3MB&	\textbf{92.28\%}	&99.51\% \\ \hline 
Ensemble ResNet & 3(26) &	- &	Soft Averaging &   27M&	100.3MB&	\textbf{92.36\%}&	99.97\%
\\ \hline
    \end{tabular}
\end{table*}
Figure \ref{fig:8} provides graphical insight into the performance characteristics of various ResNet variants. In the scope of our study, we aim to illustrate how different ResNet configurations affect performance metrics using these visualisations.
\begin{figure}
    \centering
    \includegraphics[width=1\linewidth]{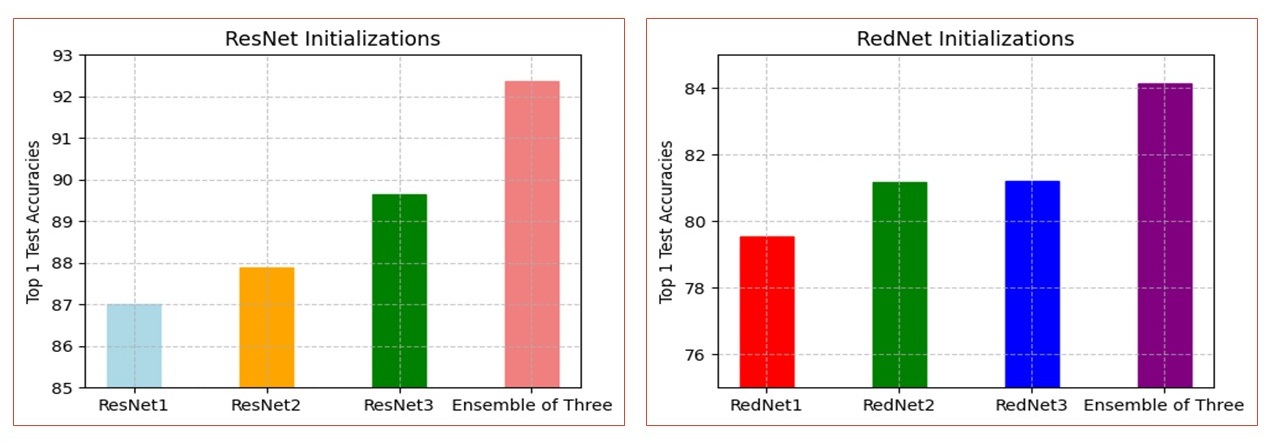}
    \caption{Accuracy comparisons of ResNet and RedNet Initialisations}
    \label{fig:8}
\end{figure}
As demonstrated by Figure \ref{fig:three_images}, both teacher ResNet and RedNet models exhibit complementary inductive biases. It is evident from Figure \ref{fig:three_images} that models are capable of capturing distinct patterns and capturing them in a variety of ways. The unique strengths of each model can explain the robustness and flexibility of the distilled knowledge acquired during the study in terms of inductive bias.
\begin{figure}
    \centering
    \includegraphics[width=\linewidth]{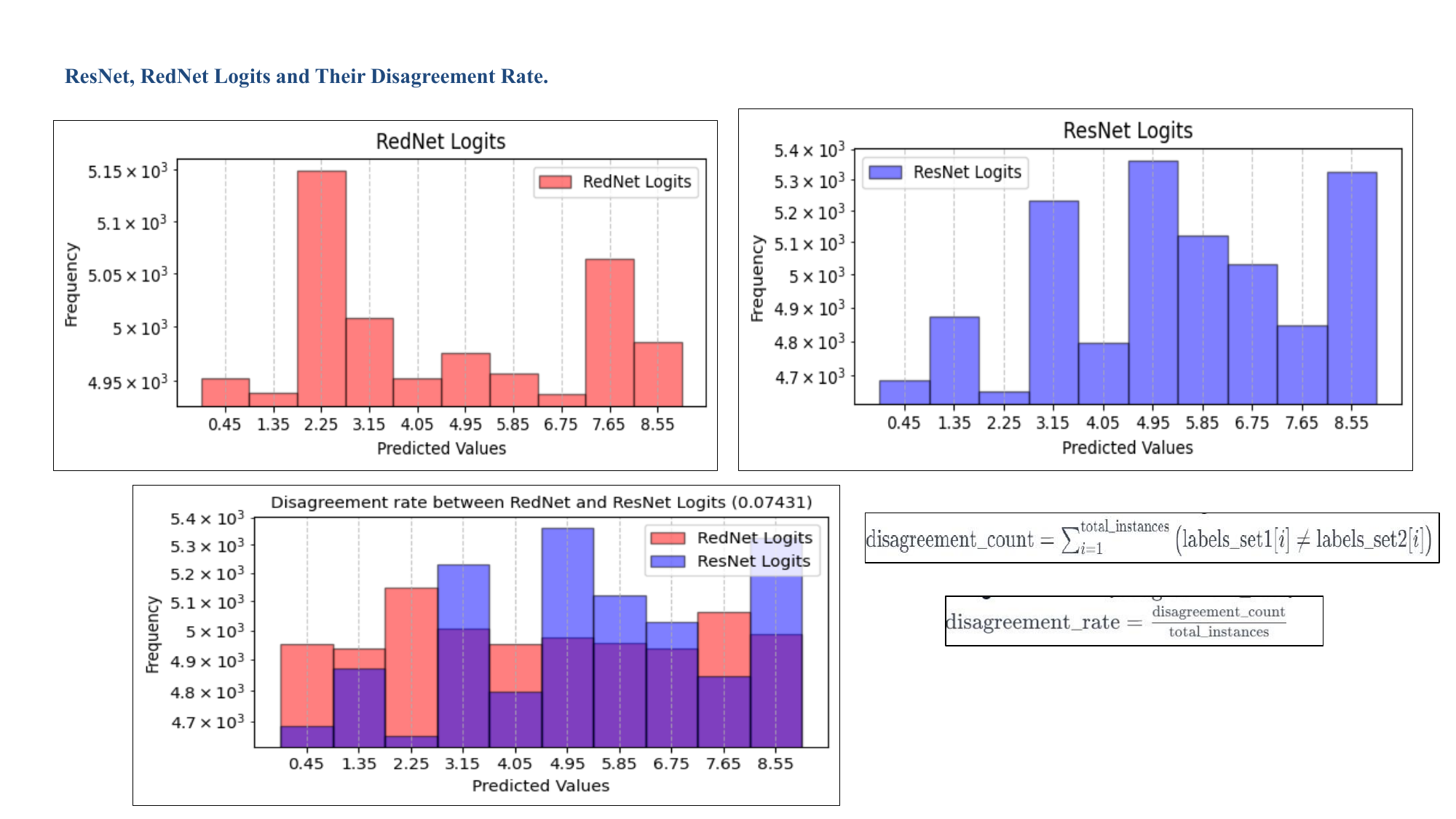}
        \caption{Showing a disagreement between ResNet and RedNet}
    \label{fig:three_images}
\end{figure}
Notably, the ResNet and RedNet ensembles exhibit a calculated disagreement rate of 0.07431. Based on the disagreement rate, this prediction discrepancy shows that these models are learning diverse and non-redundant patterns. As a result of this diversity in known patterns, each model demonstrates a unique inductive bias, further reinforcing the value of the ensemble approach for capturing a broader range of knowledge and insights. This section presents the knowledge. 

Distillation results were achieved using an ensemble of complementary teacher models, including ResNet and RedNet. This ensemble approach uses the unique inductive biases of these models to facilitate comprehensive knowledge transfer to the student model.
\begin{tiny}
\begin{table*}
    \centering
\caption{Test Accuracy comparison using an ensemble of multiple teachers, with 1000 number of Epochs and Dataset as CIFAR-100, using True Label} \label{tab:-4}
    \begin{tabular}{|m{2.9cm}|m{3.28cm}|m{1cm}|m{0.7cm}|m{0.7cm}|m{0.3cm}|m{1cm}|m{1cm} |m{1.8cm}|m{1.8cm}|}
    \hline 
Method &Teacher Type &Dist. Type&Mixup&CutMix&	$\alpha $ &\#Params&Param. Size &	Top 1 Test Accuracy&Top 5-test Accuracy
 \\ \hline 
 Student Baseline &	None&	 	 	\ding{55} &	\ding{55}&	\ding{55}&	\ding{55} &	5.4 &	\ding{55}&		67.19\% &	88.01\% \\ \hline 
 \multicolumn{10}{|c|}{KD From Single Teachers} \\
 \hline 
\multirow{2}{*}{}  
KD Baseline Single & Single ResNet &    Hard &	\ding{55}	& \ding{55} &	\ding{55}	& 9&	34 &	 68.42\% &	88.12\%
\\ \cmidrule{2-10} 
 & Single Rednet &	   Hard &	 \ding{55}	&\ding{55} &	0.5 &	7&	28&  69.43\% &	88.28\% \\
\hline   
\multicolumn{10}{|c|}{KD from an ensemble of two Teachers} \\ 
 \hline 
 KD Baseline from Ensemble of 2 Teachers&	Two ResNets&  	Hard & \ding{55}	& \ding{55} &	0.5&	18&	68&	 68.82\%&	91.77\% 
 \\ \hline 
 \multicolumn{10}{|c|}{KD from an ensemble of 3 Teachers}
\\ \hline 
\multirow{4}{*}{ } & Ensemble (3RedNets) &	  	   Hard &	\ding{55} &	\ding{55} &	0.5&	21&	82&	 69.93\% &	89.98\% \\
\cmidrule{2-10} 
& Ensemble (3ResNets) &	  	   Hard	& \ding{55} &	\ding{55}&	0.5 &	27&	100&	 69.66\% &	89.26\% \\ \cmidrule{2-10} 
{KD Ensemble} {Base Line}&Ensemble (2RedNets + 1ResNet) &	  	   Hard & \ding{55} &	\ding{55}&	0.5 &		23&	156&	 69.99\% &	90.02\% \\ \cmidrule{2-10}
 & Ensemble (2ResNets + 1RedNet) &	  	   Hard & \ding{55} &	\ding{55}&	0.5 &		25&	 228&  69.95\% &	 91.11\%
\\ \hline	
\multicolumn{10}{|c|}{KD from an ensemble of Four Teachers} \\ \hline 
KD Ensemble Baseline from 4 teachers & 	
Ensemble of 4 ResNets &	 	 	   Hard & \ding{55} &	\ding{55}&	0.5 &	36&	136&	69.27\% &	90.89\% \\ \hline
\multicolumn{10}{|c|}{Knowledge Distillation is Superior with aggressive augmentation enabled}
\\ \hline 
Student Superior &	None &	  None & \ding{51} &	\ding{51}& \ding{55}  & 5.4M &	  \ding{55}&	   73.75\% &	94.50\%
\\ \hline 
\multirow{2}{*}{KD Baseline}
& Single ResNet &   Hard & \ding{51}	&\ding{51}	& 	0.5&	9M&	34MB&	76.24\% &	 96.62\%
\\ \cmidrule{2-10}
{Single Superior}& Single Rednet	&   Hard & \ding{51}	&\ding{51}	& 	0.5&7M	&28MB &	 76.55\% &	94.25\%
\\ \hline
\multirow{4}{*}{} & Ensemble (3RedNets) &   Hard & \ding{51}	& \ding{51}	& 	0.5& 21M &	82MB&  77.13\% &97.38\%
\\ \cmidrule{2-10}
{KD} & Ensemble (3ResNets)  & Hard & \ding{51}	&\ding{51}	& 	0.5& 27M	& 100MB &  76.92\%	& 97.17\%
\\ \cmidrule{2-10}
{Ensemble Superior}&Ensemble (2RedNets + 1ResNet) &  Hard & \ding{51}	&\ding{51}	& 	0.5& 23M &	156MB	&   78.64\% &	98.01\%
\\ \cmidrule{2-10}
&Ensemble (2ResNets + 1RedNet) &  Hard & \ding{51}	&\ding{51}	& 	0.5& 25M	& 228MB	&  77.63\%	& 97.15\% \\ \hline 
 \end{tabular}
\end{table*}
\end{tiny}
\noindent
According to the results presented in Table \ref{tab:-4}, knowledge distillation and augmentation strategies can substantially enhance the performance of a student neural network on the CIFAR-100 dataset. The experiments demonstrate that inductive bias is critical in guiding effective learning. A lightweight teacher model with few parameters can significantly enhance students' accuracy. As a result of the aggregated knowledge from multiple teacher models, ensemble knowledge distillation further amplifies the improvements. The combination of aggressive augmentation techniques with knowledge distillation yields remarkable results, highlighting their potential for achieving state-of-the-art results. The above highlights the importance of a teacher's ability to transfer valuable knowledge and their inductive bias, especially when attempting to generalise a model in practice.

From Table \ref{tab:-4}, it is also worth noting that an ensemble of three performs far better than an ensemble of two or four. Ensembles achieve a bias-variance trade-off by reducing both bias and variance. Only two models might reduce variance to less than three, potentially leading to overfitting. Adding a fourth model might make the ensemble too complex, increasing overfitting and variance. Also, a certain point in ensemble learning is reached where the returns diminish. The marginal performance improvement becomes less significant after a certain number of models and may not justify the additional complexity and resource usage. 

Consequently, these findings are of paramount importance for model compression and transfer learning since they demonstrate how a modestly sized teacher model, with a well-structured structure, can impart valuable insights, guiding students towards both efficiency and generalisation, paving the way for real-world applications utilising deep neural networks. As a result, more responsive and resource-efficient AI systems can be created, allowing a more comprehensive range of applications to be developed where real-time constraints are critical.   For various practical scenarios, the ability to obtain competitive accuracy with smaller, more efficient models is one of the most crucial developments. The computational resources and latency constraints are significant considerations in real-time applications, such as object recognition on edge devices, autonomous vehicles, or mobile apps using the Internet of Things.
\section{State-of-the-Art (SOTA) Comparison}
Based on the model comparison presented in Fig. \ref{fig:a1}, DDeIT-Tiny is highly efficient and accurate, achieving an impressive accuracy of 79.00\%. A remarkable feature is that it approaches the accuracy of its teacher model, RedNet, with 79.55\% accuracy, showing the effectiveness of knowledge distillation. Having this kind of efficiency is crucial for resource-constrained or real-time applications. It offers a compelling solution for scenarios requiring accuracy and efficiency, balancing model complexity and performance. Compared to other DeiT models, the DDeIT-Tiny distinguishes itself not only by its impressive accuracy and efficiency but also by its lightweight nature. Unlike its counterparts, such as M-ViT \cite{54} and T-ViT, DeIT-Tiny has a relatively small number of parameters, whereas M-ViT and T-ViT \cite{55} have much larger model sizes. It is exceptionally lightweight as a result of this substantial parameter reduction. Its high performance, efficiency, and reduced complexity make it an ideal choice for applications with limited computational resources, further enhancing its position as a standout DeIT-Tiny model.
\begin{figure}
    \centering
    \includegraphics[width=0.8\linewidth]{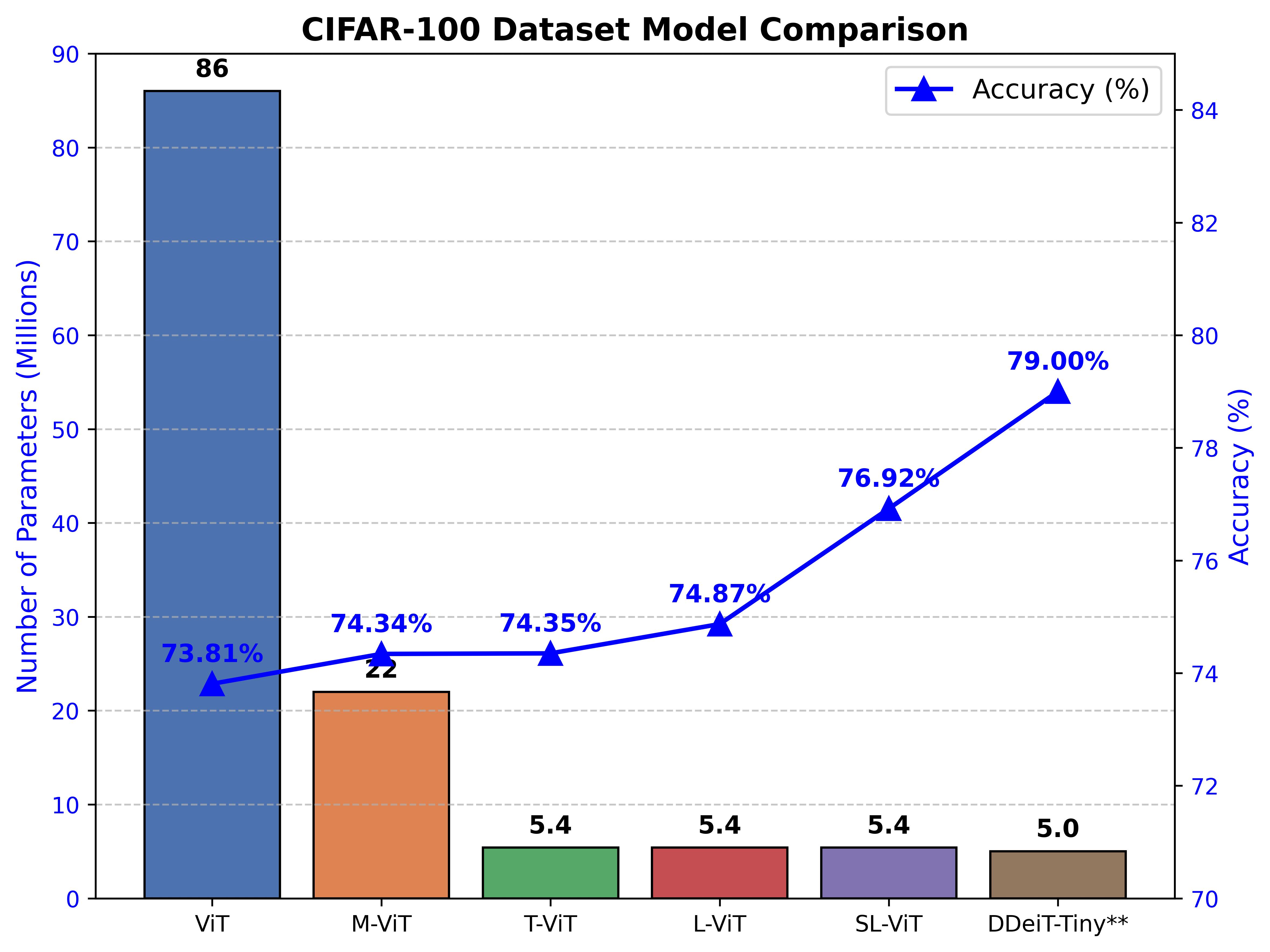}
    \caption{Comparison of Proposed DeiT-Tiny (\textbf{LIB-KD}) with SoTA Models.}
    \label{fig:a1}
\end{figure}
\begin{figure}
    \centering
    \includegraphics[width=\linewidth]{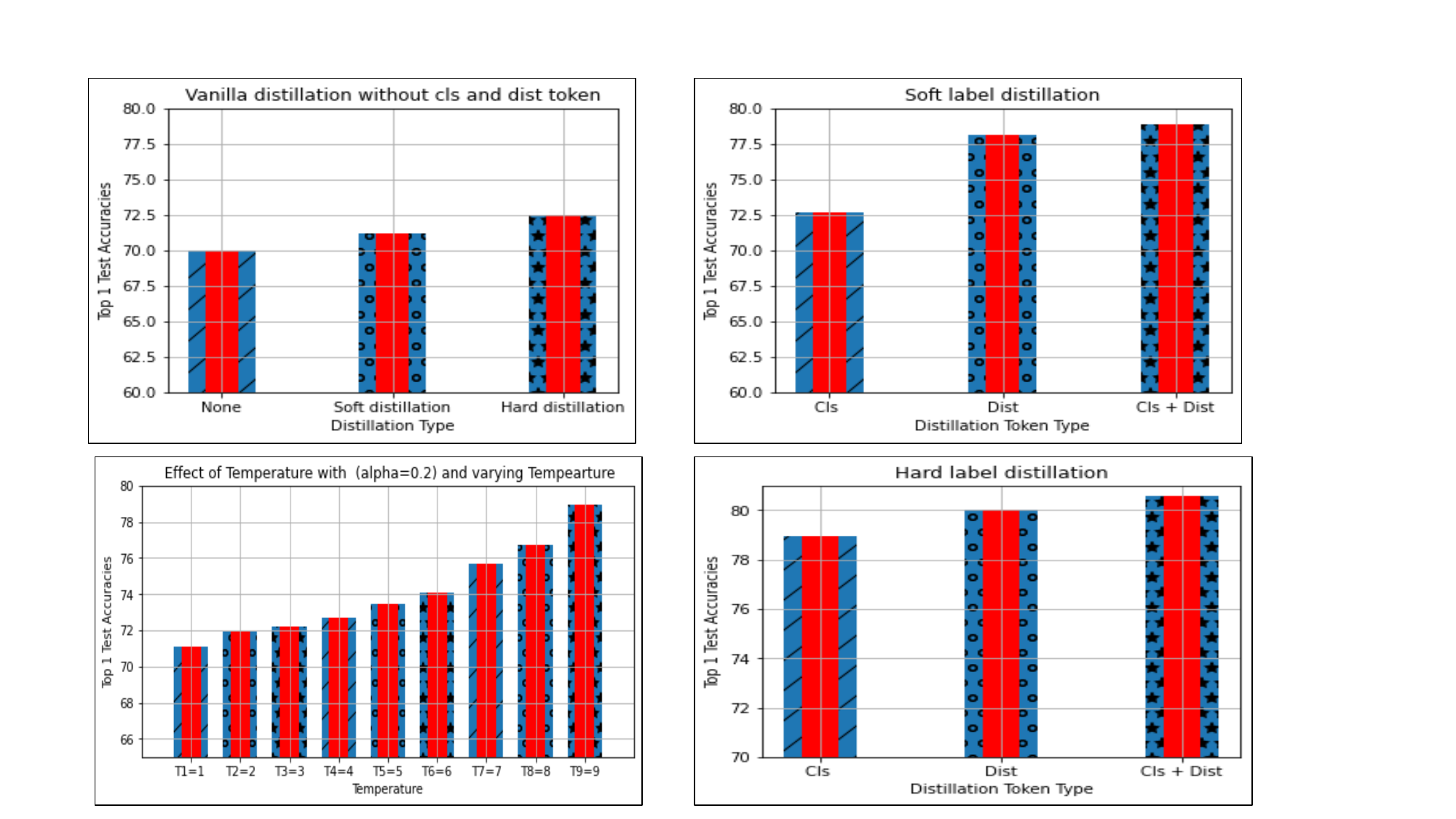}
    \caption{(a) Comparison of hard vs. soft distillation. (b) Soft with Cls+Dist tokens shows combined benefits. (c) Hard with Cls+Dist beats soft in DeiT. (d) Higher temp makes soft match complex label distill.}
  \label{fig:three_images(a)}
\end{figure}
\section{Ablation Study}
It is our primary objective in this ablation study, presented in Figure \ref{fig:distillation_results}, to evaluate the efficacy of using pre-trained heavyweight teachers on the ImageNet dataset to distill knowledge into various variants of DeiT (Data-efficient Image Transformer) models, each with a different number of parameters. Moreover, this analysis is extended to a different dataset featuring high-resolution images focused explicitly on flower classification \cite{54}, in contrast to the CIFAR-100 dataset. Our goal is to gain comprehensive insights into the effects of teacher choice, model complexity, and dataset variation on the performance and efficiency of student models by systematically evaluating the knowledge transfer process from these pre-trained teachers \cite{54} to diverse DeiT variants. This study provides valuable insight into the optimal knowledge distillation strategy for real-world applications and domains with varied data characteristics.
\begin{figure}[htbp]
    \centering
    \includegraphics[width=\linewidth]{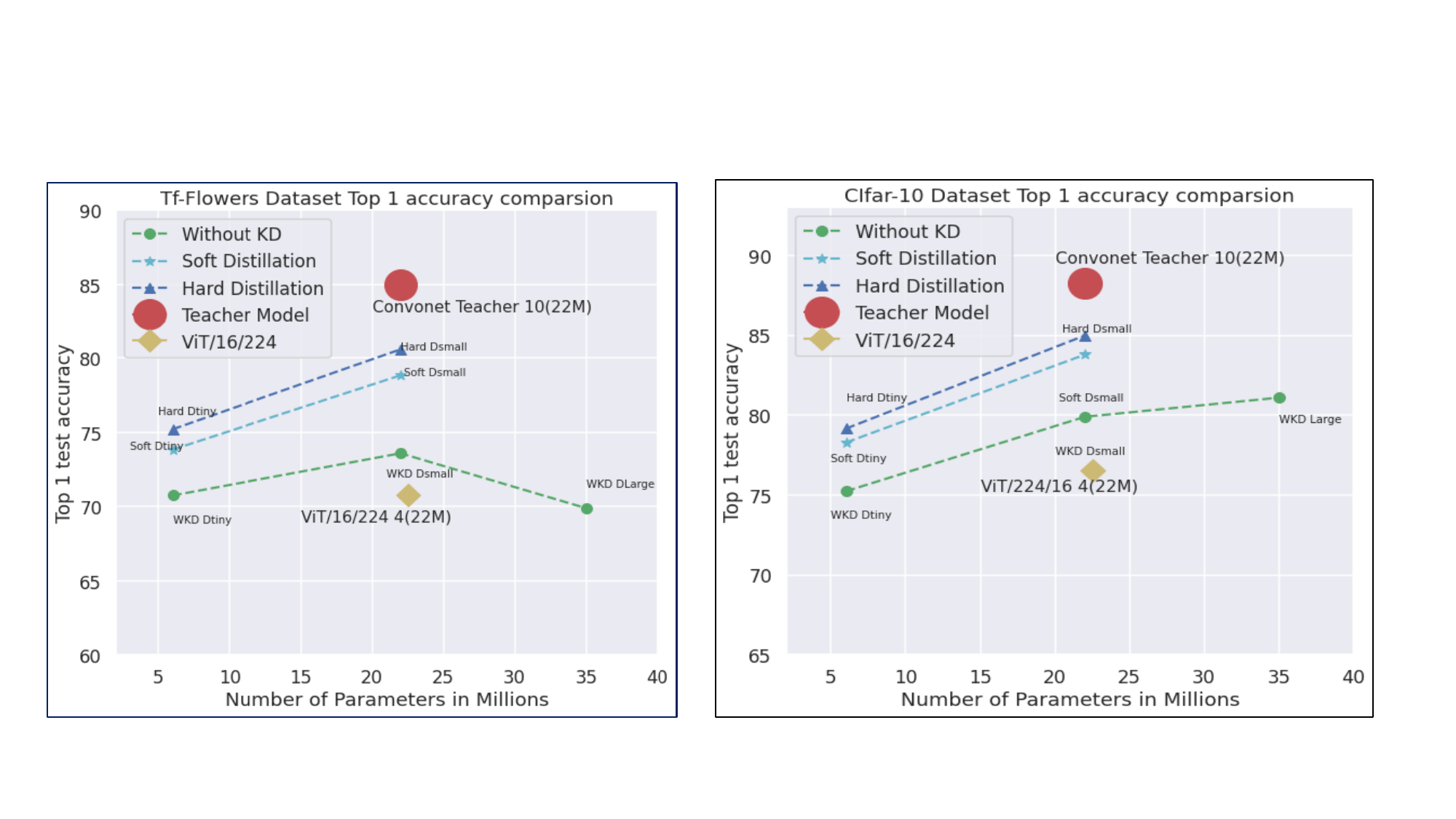}
    \caption{Distillation on (a) TF-Flowers \cite{54} and (b) CIFAR-10 \cite{53}: Teacher ResNet-220M vs. student DeiT-Tiny-88M ($\text{~} 4\times $ lighter)}
    \label{fig:distillation_results}
\end{figure}
Compared to the TF Flowers dataset (Figure \ref{fig:distillation_results}a), the CIFAR-10 dataset (Figure \ref{fig:distillation_results}b) shows substantially improved performance, primarily due to the larger volume and diversity of the data. With abundant data, CIFAR-10 enables better generalization by capturing complex patterns and nuances, thus improving accuracy. Additionally, aggressive augmentation techniques applied to the TF Flowers dataset further enhance model performance.
\section{Conclusion}
The paper concludes by discussing the transformative potential of ViTs in computer vision, emphasising their ability to bridge the gap between visual and textual domains. The report identifies a fundamental challenge in ViTs: their lack of inherent inductive biases, which leads to their heavy reliance on large datasets. As the paper systematically analyses ViTs and CNNs, we emphasise the uniformity of ViT representations across layers and the crucial role of local receptive fields. In addition, it underscores the need for innovative solutions to allow ViTs to collect local feature information while retaining their global receptive field capabilities. The study employs an ensemble-based approach that leverages knowledge from complementary multi-teacher models (INNs and CNNs) to address these challenges. The method optimises ViT performance and efficiency, breaking new ground when applying transformers to small datasets with a diverse ensemble of lightweight teachers.
\bibliographystyle{spmpsci}      % mathematics and physical sciences
\bibliography{ref}   % name your BibTeX data base

%\begin{thebibliography}{}
% etc
%\end{thebibliography}

\end{document}